\documentclass[twoside,11pt]{article}

\usepackage{blindtext}

\usepackage{jmlr2e}

\usepackage{amsmath}

\usepackage{subcaption}
\usepackage{graphicx}
\usepackage{hyperref}
\usepackage{url}
\usepackage{algorithm}
\usepackage{algpseudocode}
\newtheorem{assumption}{Assumption}

\DeclareMathOperator*{\argmin}{arg\,min}

\ShortHeadings{Primal-Only Actor Critic Algorithm for Robust Constrained Average Cost MDPs}{Satheesh, Sathish, Ganesh, Powell, and Aggarwal}
\firstpageno{1}

\begin{document}

\title{Primal-Only Actor Critic Algorithm for Robust Constrained Average Cost MDPs}

\author{\name Anirudh Satheesh$^*$ \email anirudhs@terpmail.umd.edu \\
       \addr 
       University of Maryland
       \AND
       \name Sooraj Sathish$^*$ \email Sooraj.Sathish@iiitb.ac.in \\
       \addr IIIT Bangalore
       \AND
       \name Swetha Ganesh \email ganesh49@purdue.edu \\
       \addr Purdue University
       \AND
       \name Keenan Powell \email kpowell1@terpmail.umd.edu \\
       \addr University of Maryland
       \AND
       \name Vaneet Aggarwal \email vaneet@purdue.edu \\
       \addr Purdue University}

\maketitle

\def\thefootnote{$^*$}%
\footnotetext{These authors contributed equally.}%
\def\thefootnote{\arabic{footnote}}%

\begin{abstract}
  In this work, we study the problem of finding robust and safe policies in Robust Constrained Average-Cost Markov Decision Processes (RCMDPs). A key challenge in this setting is the lack of strong duality, which prevents the direct use of standard primal-dual methods for constrained RL. Additional difficulties arise from the average-cost setting, where the Robust Bellman operator is not a contraction under any norm. To address these challenges, we propose an actor-critic algorithm for Average-Cost RCMDPs. We show that our method achieves both \(\epsilon\)-feasibility and \(\epsilon\)-optimality, and we establish a sample complexities of \(\tilde{O}\left(\epsilon^{-4}\right)\) and \(\tilde{O}\left(\epsilon^{-6}\right)\) with and without slackness assumption, which is comparable to the discounted setting.
\end{abstract}
\section{Introduction}
Reinforcement Learning (RL) has achieved remarkable success across domains such as robotics \citep{chen2023option}, transportation \citep{al2019deeppool},  and large language model fine-tuning \citep{gaur2025sample}. However, most approaches assume that training and deployment occur under identical conditions. In practice, real-world experiments are costly and risky, necessitating reliance on simulators. Additionally, even the most detailed simulators cannot fully capture the variability, noise, and stochasticity of real-world environments. This mismatch, known as the sim2real gap, can lead to severe performance degradation and, in safety-critical systems, catastrophic failures or equipment damage. 

Many real-world applications also impose strict safety or resource constraints: autonomous vehicles must guarantee human safety, communication networks must respect bandwidth limits, and transportation systems must meet time constraints, and the stochastic nature of RL policies makes consistently satisfying these constraints challenging. The need for both robustness against environmental shifts and adherence to strict constraints motivates the Robust Constrained Markov Decision Process (RCMDP) framework, where policies must guarantee worst-case performance under an uncertainty set of transitions while satisfying constraints. 

In the RCMDP framework, distributional robustness is modeled by defining an uncertainty set of environments that captures potential distribution shifts in transition dynamics, with the objective of optimizing the worst-case performance within this set. Constraint satisfaction is incorporated by augmenting the reward function with additional constraint functions that must remain below specified thresholds. Our setting adds further complexity by focusing on the infinite-horizon average reward, rather than the discounted return. This formulation is more suitable for capturing long-term objectives and is particularly relevant in applications requiring persistent and consistent performance over extended time horizons.

The literature on robust and constrained MDPs is still limited, especially in the average-reward case. Strong duality does not hold in RCMDPs \citep{wang2022robust, ma2025rectified}, which prevents extending sample-efficient primal-dual algorithms to the robust MDP formulation. Thus, recent work has opted for primal-only methods. \citet{kitamura2024near} propose an epigraph formulation for discounted reward RCMDP, but the necessity of binary search increases the sample complexity to \(\tilde{O}(\epsilon^{-6}\log(\epsilon^{-1}))\). Additionally, this work requires perfect estimation of the robust value function, which may not be tractable for large state spaces. \citet{ma2025rectified} and \citet{ganguly2025efficient} develop primal-only algorithms that achieve sample complexities of \(O(\epsilon^{-6})\), while also in the discounted setting. 

\paragraph{Challenges and Contributions}
Two primary challenges motivated our specific solution method. Firstly, the average reward setting's Bellman operator does not satisfy a trivial contraction property like we have in the discounted setting. Secondly, most works on Constrained MDPs utilize the Primal-Dual method and propose algorithms that alternatively update the respective Lagrangian multipliers. However, \cite{wang2022robust,ma2025rectified} shows that strong duality does not hold in the robust setting, which motivates the choice of primal-only algorithms.

We summarize our work and contributions as follows:
\begin{itemize}
    \item We present the first formulation and analysis of the Average Cost RCMDPs, extending beyond the discounted reward setting and addressing the major problem of the lack of strong duality.
    \item We propose an actor critic algorithm that theoretically guarantees $\epsilon$-feasibility and optimality for different uncertainty sets (Contamination, TV Distance, Wasserstein).
    \item We show sample complexity guarantees of \(O(\epsilon^{-4})\) with the slackness assumption and \(O(\epsilon^{-6})\) without the slackness assumption.
\end{itemize}
\section{Related Work}\label{section: related-works}
\subsection{Constrained Reward MDPs}
Constrained Reward MDPs (CMDPs) have been studied extensively in the literature, both in the discounted reward and average reward setting. The model-based approaches \citep{chen2022learning,agarwal2022regret,agarwal2021concave} construct estimates of the transition probabilities and then derive \textit{safe} policies. These approaches often involve continuously solving Linear Programs (LPs) as estimated models are updated, leading to computational inefficiency and a need for substantial memory. Model-free algorithms \citep{ wei2022provably, bai2024learning,xu2025global} learn the optimal policy or value function directly from sampling of the environment. This is generally more compute efficient and requires less memory. 
Owing to the clear advantages and the real-world applicability of model-free algorithms, we also focus our attention on this setting. 

Constrained RL problems have been addressed using various model-free solution methods, but the most common approach has been the primal-dual method \citep{altman2021constrained,paternain2022safe,bai2022achieving,wang2022robust, bai2023achieving,bai2024learning,mondal2024sample,xu2025global}. Here, the constrained problem is converted into its dual formulation, where the objective is a weighted sum of the reward and the constraints. These weights are Lagrangian multipliers, which are updated alternatively until convergence. \cite{paternain2019constrained} show that strong duality holds in the non-robust constrained RL setting, and this primal-dual method attains zero duality gap. The less-studied counterpart are the works on primal-only solutions \citep{dalal2018safe,liu2020ipo,yang2020projection}. These works ensure that the constraints are not violated (or violation is bounded) without the use of Lagrange multipliers. For example, CRPO \citep{pmlr-v139-xu21a} ensures convergence to an optimal safe policy by only updating the reward when no constraint is violated. They leverage a novel combinatorial bucketing approach to show the convergence even when the objective being optimised switches every iteration. 
Since strong duality does not hold in the distributionally robust setting \citep{wang2022robust,ma2025rectified}, we look to design a primal-only algorithm. 

\begin{table*}[ht]
\centering
\resizebox{\textwidth}{!}{%
\begin{tabular}{|c|c|c|}
\hline
\textbf{Method} & \textbf{Setting} & \textbf{Sample Complexity} \\
\hline
\citep{xu2025global} & Constrained, Average & $O(\epsilon^{-2})$\\
\hline
\citep{li2022first} & Robust, Discounted & $\widetilde{O}(\epsilon^{-2})$ \\
\hline
\citep{xu2025efficient} & Robust, Average & $\widetilde{O}(\epsilon^{-2})$ \\
\hline
\citep{kitamura2024near} & Robust, Constrained, Discounted & $\widetilde{O}(\epsilon^{-6})$ \\
\hline
\cite{ma2025rectified} & Robust, Constrained, Discounted & $O(\epsilon^{-6})$ \\
\hline
\cite{ganguly2025efficient} (w/o Slackness) & Robust, Constrained, Discounted & $O(\epsilon^{-6})$ \\
\hline
\cite{ganguly2025efficient} (w/ Slackness) & Robust, Constrained, Discounted & $O(\epsilon^{-4})$ \\
\hline
Our work (w/o Slackness) & Robust, Constrained, Average & ${O}(\epsilon^{-6})$  \\
\hline
Our work (w/ Slackness) & Robust, Constrained, Average & ${O}(\epsilon^{-4})$  \\
\hline
\end{tabular}
}
\caption{Comparison of sample complexities of different methods to solve Robust and Constrained MDPs. Our work achieves state of the art sample complexity over existing robust constrained MDP methods.}
\label{tab:complexity_comparison}
\end{table*}

\subsection{Robust RL}
Dynamic programming approaches to solve model-based robust RL problems have been explored extensively in the past \citep{NIPS2003_300891a6,Iyengar2005-aj, wiesemann2013robust, tamar2014scaling}. More recent studies on the discounted reward setting have focused their efforts on problems where the uncertainty set is unknown and only samples from the nominal distribution can be collected \citep{zhou2021finite,panaganti2022sample,wang2022robust,wang2023policy}. 

Most studies on robust RL have primarily considered the infinite-horizon discounted reward setting, where the Bellman operator always satisfies a contraction property it inherits from the discount factor. Since no such trivial contraction property exists in the average reward setting, the  literature on robust average reward \citep{wang2023robustaverage, chen2025sample}  consider approaches where the results from discounted reward could be converted to that for average reward, while in the absence of constraints. Some average-reward works explore model-free solutions through Halpern iteration \citep{roch2025finite}, while others exploit ODE methods in stochastic approximation to prove convergence \citep{wang2023model}. 

More recently, a novel semi-norm with the contraction property has been found \citep{xu2025finite}, and a corresponding Actor-Critic approach to robust average reward unconstrained RL has been proposed \citep{xu2025efficient}.
In our work, we leverage this semi-norm and propose a similar Actor-Critic approach to Robust Constrained Average cost RL.

\subsection{Robust Constrained MDPs}
The literature on robust  constrained MDPs (RCMDPs) is limited because strong duality does not hold in this setting \citep{wang2022robust}. Some studies \citep{russel2020robust,mankowitz2020robust,wang2022robust, zhang2024distributionally} have tried to address this problem through primal-dual methods by quantifying and tracking the duality gap or restricting to certain policy classes that satisfy strong duality. However, these works do not provide explicit iteration and sample complexity guarantees. 

More recently, \citet{kitamura2024near} propose an epigraph formulation of the discounted primal problem and provide explicit sample complexity guarantees. Unfortunately, the binary search employed in this solution elicits a very high sample complexity. Furthermore, it is known that the binary search approaches fail when the robust value estimates are noisy \citep{horstein2003sequential}.

To address the above shortcomings, \citet{ganguly2025efficient} propose a unique formulation of the discounted RCMDP problem without the use of epigraphs and binary search. They show $\epsilon$-feasibility and optimality of their mirror-descent algorithm, offering improved iteration complexity guarantees.
A parallel work by \cite{ma2025rectified} takes inspiration from CRPO (\cite{pmlr-v139-xu21a}) and achieves the same sample complexity guarantees as \cite{ganguly2025efficient}. 

We emphasize that, to the best of our knowledge, this is the first work to address the RCMDP problem in the average reward/cost setting and provide optimal sample complexity guarantees. Table \ref{tab:complexity_comparison} is a concise comparison of the various methods in the literature and demonstrates the optimal performance of our algorithm. 


\section{Formulation}
\label{sec: Formulation}
\subsection{Robust Average Cost MDPs}
\label{subsec: average cost MDPs formulation}
An infinite horizon robust average cost Markov Decision Process (MDP) can be defined by the tuple \((\mathcal{S}, \mathcal{A}, r, \mathcal{P}, \rho)\), where \(\mathcal{S}\) is the state space, \(r: \mathcal{S} \times \mathcal{A} \to [0, 1]\) is the cost function, \(\mathcal{P}\) is an uncertainty set of transition kernels, and \(\rho: \mathcal{S} \to [0, 1]\) is the initial state distribution. At each timestep, a transition kernel \(P\) is randomly selected from \(\mathcal{P}\) and is used to transition the environment to the next state. We focus on the \((s, a)\)-rectangular uncertainty set \(\mathcal{P} = \otimes_{s, a} \mathcal{P}^a_s\) \citep{NIPS2003_300891a6, Iyengar2005-aj}, where 
\begin{align}
    \mathcal{P}^a_s = \{P \in \Delta(\mathcal{S}): D(P, P^{\circ}) \leq R\}
\end{align}
and \(P^\circ\) is the nominal transition kernel. The goal of the policy \(\pi: \mathcal{S} \to \Delta(\mathcal{A})\) is to maximize the worst-case average cost over the set of transitions \(\mathcal{P}\)
\begin{align}
    \label{eqn: time varying model average cost objective}
    g_\mathcal{P}^\pi(s) = \max_{\kappa \in \otimes_{k \geq 0} \mathcal{P}} \lim_{T \to \infty} \mathbb{E}_{\kappa, \pi} \left[\frac{1}{T} \sum_{t=0}^{T-1} r_t \bigg| s_0 = s \right]
\end{align}
\citet{JMLR:v25:23-0526} showed this objective is the same under the stationary model
\begin{align}
    \label{eqn: stationary model average cost objective}
    g_\mathcal{P}^\pi(s) = \max_{P \in  \mathcal{P}} \lim_{T \to \infty} \mathbb{E}_{P, \pi} \left[\frac{1}{T} \sum_{t=0}^{T-1} r_t \bigg| s_0 = s \right]
\end{align}
Thus, we focus solely on the stationary case. We denote the maximizers of Eq (\ref{eqn: stationary model average cost objective}) as the worst-case transition kernels and \(\Omega_g^\pi = \{P \in \mathcal{P}: g_P^\pi = g_{\mathcal{P}}^\pi\}\), where
\begin{align}
    g_P^\pi = \lim_{T \to \infty} \mathbb{E}_{P, \pi} \left[\frac{1}{T} \sum_{t=0}^{T-1} r_t \bigg| s_0 = s \right]
\end{align}
is the average cost of \(\pi\) with transition kernel \(P\).

We also focus on the model-free setting, where samples can only be accessed from the nominal transition kernel \(P^\circ\). We are interested in estimating both the robust value function \(V_{P_V}^\pi\) and the robust average cost \(g_{P_V}^\pi\). The robust value function can be defined through the robust Bellman equation in Theorem~\ref{thm: robust bellman equation}.

\begin{theorem}[Robust Bellman Equation, Theorem 3.1 in \citep{wang2023model}]
    \label{thm: robust bellman equation}
     If \((g, V)\) is a solution to the robust Bellman equation
     \begin{align}
        \label{eqn: robust bellman equation}
         V(s) = \sum_a \pi(a | s)(r(s, a) - g + \sigma_{\mathcal{P}^a_s}(V)), \forall s \in \mathcal{S}
     \end{align}
     where \(\sigma_{\mathcal{P}^a_s} = \min_{P \in \mathcal{P}^a_s}\) is denoted as the support function, then the scalar \(g\) corresponds to the robust average cost, i.e., \(g = g_\mathcal{P}^\pi\), and the worst-case transition kernel \(P_V\) belongs to the set of minimizing transition kernels, i.e., \(P_V \in \Omega_\mathcal{P}^\pi\) where \(\Omega_g^\pi = \{P \in \mathcal{P}: g_P^\pi = g_{\mathcal{P}}^\pi\}\). Furthermore, the function \(V\) is unique up to an additive constant, where if \(V\) is a solution to the Bellman equation, then we have \(V = V_{P_V}^\pi + c\mathbf{e}\), where \(c \in \mathbb{R}\) and \(\mathbf{e}\) is the all-ones vector in \(\mathbb{R}^{|\mathcal{S}|}\), and \(V_{P_V}^{\pi}\) is defined as
the relative value function of the policy \(\pi\) under the single transition \(P_V\) as follows:
\begin{align}
    V_{P_V}^{\pi}(s) = \mathbb{E}_{P_V, \pi} \left[\sum_{t=0}^\infty (r_t - g_{P_V}^\pi) \bigg| s_0 = s\right] 
\end{align}
\end{theorem}

Using Theorem~\ref{thm: robust bellman equation}, we can define \(\sigma_{\mathcal{P}^a_s}(V)\) for different uncertainty sets.

\paragraph{Contamination Uncertainty Set}
The R-contamination uncertainty set is \(\mathcal{P}^a_s = \{(1 - R)P^\circ_{s, a} + Rq \,| \, q \in \Delta(\mathcal{S})\}\), where \(R \in (0, 1)\) is the radius of the uncertainty set. The support function of the R-contamination set can be directly computed as
\begin{align}
    \label{eqn: contamination sigma}
    \sigma_{\mathcal{P}^a_s}(V) = (1-R)P^{\circ}_{s, a}V + R \max_{s} V(s)
\end{align}
We can also use this formulation to construct the estimator of the worst case transition effect
\begin{align}
    \label{eqn: contamination estimator}
    \hat{\sigma}_{\mathcal{P}^a_s}(V) = (1-R)V(s') + R \max_{x} V(x)
\end{align}
where \(s'\) is the next state using the nominal transition kernel.

\paragraph{Total Variation Uncertainty Set}
The total variation uncertainty set is \(\mathcal{P}^a_s =\) \(\\ \left\{\frac{1}{2}\|q - P^\circ_{s, a}\|_1 \leq R \,|\, q \in \Delta(\mathcal{S})\right\}\). We can define the support function using its dual formulation 
\begin{align}
    \label{eqn: total variation sigma}
    \sigma_{\mathcal{P}^a_s}(V) = \min_{\mu \geq \mathbf{0}} \left(P^{\circ}_{s, a}(V - \mu) - R\|V - \mu\|_{\text{sp}}\right)
\end{align}
where \(\|\cdot\|_{\text{sp}}\) is the span semi-norm \citep{Iyengar2005-aj}.

\paragraph{Wasserstein Uncertainty Sets} 
We consider the \(l\)-Wasserstein distance \(W_l(q, p) = \inf_{\mu \in \Gamma(p, q)}\|d\|_{\mu, l}\), where \(l \in [1, \infty)\), \(p, q \in \Delta(\mathcal{S})\), \(\Gamma(p, q)\) is the distributions over \(\mathcal{S} \times \mathcal{S}\) with marginal distributions \(p, q\), and \(\|d\|_{\mu, l} = \left(\mathbb{E}_{(X, Y) \sim \mu}\left[d(X, Y)^l\right]\right)^{\frac{1}{l}}\). The Wasserstein distance uncertainty set is then defined as \(\mathcal{P}^a_s = \{W_l(P^\circ_{s, a}, q) \leq R \,|\, q \in \Delta(\mathcal{S})\}\). Then we can define the support function for the Wasserstein uncertainty set \citep{gao2023distributionally} as
\begin{equation}
    \label{eqn: wasserstein sigma}
    \sigma_{\mathcal{P}^a_s} = \inf_{\lambda \geq 0} \left(-\lambda \delta^l + \mathbb{E}_{s \sim \mathcal{S}, a \sim \pi(s)}\left[\sup_{y}V(y) + \lambda d(\mathcal{S}, y)^l\right]\right)
\end{equation}

Following Theorem~\ref{thm: robust bellman equation}, we can define the robust Bellman operator in Theorem~\ref{thm: robust bellman operator}.
\begin{theorem}[Robust Bellman Operator \citep{JMLR:v25:23-0526}]
    The robust Bellman Operator \(\mathbf{T}_g\) is defined as
    \begin{align}
        \label{thm: robust bellman operator}
        \mathbf{T}_g(V)(s) = \sum_a \pi(a|s)\left[r(s, a) - g + \sigma_{\mathcal{P}^a_s}(V)\right], \forall s \in \mathcal{S}
    \end{align}
\end{theorem}
The main challenge with the robust Bellman operator is that it does not satisfy a contraction under standard norms. Thus, we leverage the contraction under the semi-norm introduced in \citet{xu2025finite} for our stochastic approximation algorithms. We also make the assumption throughout this work that the induced Markov Chain from the policy \(\pi\) is irreducible and aperiodic (Assumption~\ref{assumption: irreducibility of markov chain}). 

\begin{assumption}[Ergodicity]
    \label{assumption: irreducibility of markov chain}
    The Markov chain induced by every policy \(\pi \in \Pi\) is irreducible and aperiodic for all \(P \in \mathcal{P}\), where \(\Pi = \{\pi | \pi : \mathcal{S} \to \Delta(\mathcal{A})\}\).
\end{assumption}

Assumption~\ref{assumption: irreducibility of markov chain} is widely used in the robust average reward reinforcement learning literature \citep{wang2023model, Wang_Velasquez_Atia_Prater-Bennette_Zou_2023, NEURIPS2024_1f28e934, xu2025finite}. This ensures that from any state, it is possible to eventually reach any other state, and the system does not get stuck in deterministic cycles. This guarantees the existence of a unique stationary distribution for a given policy, which is fundamental to the average-cost setting. Under this assumption, the average cost is independent of the starting state, so we can write the robust average cost as \(g_{\mathcal{P}}^\pi\).


\subsection{Robust Constrained Average Cost MDPs}
\label{subsec: robust constrained average cost markov decision processes}
We extend robust average cost MDPs to robust constrained average cost MDPs by including \(I\) constraint functions \(c_i: \mathcal{S} \times \mathcal{A} \to [0, 1]\) and corresponding thresholds \(b_i \in \mathbb{R}\) for each \(i = 1, 2, \cdots, I\) (we keep the constraint values bounded between 0 and 1 for simplicity). Thus, the worst-case average constraint value with policy \(\pi\) on constraint \(i\) is
\begin{align}
    g_\mathcal{P}^{\pi, i} = \max_{P \in  \mathcal{P}} \lim_{T \to \infty} \mathbb{E}_{P, \pi} \left[\frac{1}{T} \sum_{t=0}^{T-1} c_{i, t} \right]
\end{align}
where \(c_{i, t}\) is the constraint value of constraint \(i\) at time \(t\). Additionally, let \(g_{\mathcal{P}}^{\pi, 0} = g_{\mathcal{P}}^{\pi}\) be the worst case average cost. The goal of the robust constrained average cost MDP is to find a policy that minimizes the worst-case average cost while ensuring each constraint is satisfied under the worst-case transition kernel:
\begin{align}
    \begin{split}
    \label{eqn: original objective}
    \pi^* = \argmin_{\pi} g_{\mathcal{P}}^{\pi, 0} \quad \text{s.t.} \\\quad g_{\mathcal{P}}^{\pi, i} \leq b_i, i = 1, 2, \cdots, I
    \end{split}
\end{align}

\paragraph{Issues with Primal-Dual methods}
Many existing works approach constrained reinforcement learning problem via primal-dual algorithms using Lagrange multipliers. However, this approach faces two fundamental obstacles in the robust setting.

First, strong duality is not guaranteed. While \citep{paternain2019constrained} established that the duality gap is often zero for standard (non-robust) CMDPs, this result relies on Slater's condition to ensure a strictly feasible policy. In the robust case, however, the set of achievable robust state-action occupancy measures (i.e., those under the worst-case models) is not necessarily convex \citep{wang2022robust}. This breakdown of the underlying convexity means that Slater's condition is no longer sufficient to guarantee a zero duality gap and allow us to use the dual formulation.
\begin{align}
    \label{eqn: robust lagrangian}
    g^{\pi^*}_{\mathcal{P}} = \min_{\lambda \in \mathbb{R}^N_+} \min_{\pi \in \Pi}  \left(g_{\mathcal{P}}^\pi + \sum_{i=1}^I \lambda_i (g_{\mathcal{P}}^{\pi, i} - b_i)\right) 
\end{align}
Secondly, the formulation of the Lagrangian is difficult to solve in the robust case due to the maximization over the transition kernels in the uncertainty set \citep{kitamura2024near}. This motivates the need for approaches that do not rely on strong duality.

\section{Proposed Algorithm: Robust Constrained Average Cost Actor Critic}
\label{sec: method}
We can avoid the non-convexity and intractability issues with primal dual methods that require strong duality by focusing solely on primal methods. For our problem formulation we take inspiration from a recent work by \cite{ganguly2025efficient} and adapt it for the average reward case,
\begin{align}\label{eq: ganguly}
    F_{\mathcal{P}}^{\pi} = \min_{\pi} \max \left\{ \frac{g_{\mathcal{P}}^{\pi, 0}}{\lambda}, \; \max_{i} \left\{g_{\mathcal{P}}^{\pi, i} - b_i + \zeta \right\} \right\}
\end{align}

 where \(\zeta\) is the slackness term. Here, the intuition is to focus on the largest constraint violation and optimize for it in each update. If no constraints are violated ($g_{\mathcal{P}}^{\pi, i} - b_i \leq 0, \forall i \in [1, \dots I]$), then we optimize for the cost function $g^{\pi, 0}_{\mathcal{P}}$. Here, $\lambda$ is introduced to regulate the trade-off between optimizing the primary cost and mitigating constraint violations. A sufficiently large $\lambda$ ensures that constraint violations cannot be ignored, while feasibility shifts the focus back to minimizing the cost objective.

 It is worth noting that our work is not a direct extension of \cite{ganguly2025efficient} to the average cost setting, as their approach relies solely on mirror descent. Our choice of an actor-critic (AC) framework is necessitated by a core challenge in the average-cost setting: the robust Bellman operator is not a contraction under standard norms. Consequently, standard gradient-based methods like Online Mirror Descent are not directly applicable, as the iterative processes needed to estimate their required Q-functions would diverge. Our AC approach resolves this directly: the critic uses a specialized algorithm that converges under a specific semi-norm (\cite{xu2025finite}), providing the stable Q-function estimates the actor requires for a provably convergent update. Furthermore, although we draw inspiration from \cite{xu2025efficient} and \cite{NEURIPS2024_1f28e934}, we cleverly utilise their results on critic estimation sample complexity and robust performance difference (respectively) for the constrained setting, which they did not tackle.

 Instead of solving a convex optimization problem in the dual formulation, we look at Eq. \eqref{eq: ganguly} and perform gradient descent in the direction of \(\nabla F\). However, \(F\) is a function of non linear robust average costs calculated by taking the maximum over all transition kernels in the uncertainty set for each constraint / objective. Thus, we cannot directly take the gradient, as this objective is not differentiable everywhere. To circumvent this, we can employ subgradient methods, which have been heavily used in non-differentiable optimization. 
\begin{definition}[Definition 3.1 in \citep{NEURIPS2024_1f28e934}]
    For any function \(f: \mathcal{X} \subseteq \mathbb{R}^N \to \mathbb{R}\), the Fréchet sub-gradient \(u \in \mathbb{R}^N\) is a vector that satisfies
    \begin{align}
        \label{eqn: frechet subgradient equation}
        \lim_{\delta \to 0} \inf_{\delta \neq 0} \frac{f(x) - f(x) - \langle u, \delta\rangle}{\|\delta\|} \geq 0
    \end{align}
\end{definition}
When \(f\) is differentiable, the subgradient of \(f\) is the same as the gradient. Leveraging subgradient methods, we can find the subgradient for the robust average cost MDP.
\begin{lemma}[Lemma 3.2 in \citep{NEURIPS2024_1f28e934}]
    \label{lemma: subgradient definition}
    Let \(d_{\mathcal{P}}^\pi\) denote the stationary distribution of the state under the worst-case transition kernel of policy \(\pi\). Denote the robust \(Q\)-function as under policy \(\pi\) as

\begin{equation}
\label{eqn: Q function definition}
\begin{split}
Q^{\pi}(s,a) = \max_{\kappa\in\otimes_{t\ge 0}\mathcal{P}}
\mathbb{E}\Bigg[ &
\sum_{t=0}^{\infty}\bigl(r(s_t,a_t)-g_{\mathcal{P}}^\pi\bigr)\Bigm| \\
&\qquad s_0 = s,\; a_0 = a,\; \pi \Bigg].
\end{split}
\end{equation}

    Then let \(\nabla g_{\mathcal{P}}^\pi\) be the subgradient of \(g_{\mathcal{P}}^\pi\), we have 
    \begin{align}
        \label{eqn: subgradient formulation}
        \nabla g_{\mathcal{P}}^\pi(s, a) = d_{\mathcal{P}}^\pi Q_{\mathcal{P}}^\pi(s, a)
    \end{align}
\end{lemma}

\begin{theorem}[Theorem 5.3 in \citep{xu2025efficient}]\label{theorem: robust-bellman-op-main}
     Let the robust \(Q\)-function  under policy \(\pi\) be defined by Eq. \eqref{eqn: Q function definition}, then \(Q^\pi\) satisfies the robust Bellman equation
     \begin{align}
        \label{eqn: q function bellman equation}
         Q^{\pi}(s, a) = r(s, a) - g^{\pi}_{\mathcal{P}} + \sigma_{\mathcal{P}^a_s}(V^\pi)
     \end{align}
     where \(V^\pi = \sum_{a} \pi(a | s) Q^\pi(s, a)\) is the robust relative value function, and \(g^\pi_{\mathcal{P}}\) is the robust average cost.
\end{theorem} 
Before  proving convergence, we need to show that solving our formulation in Eq. \eqref{eq: ganguly} is equivalent to solving for the original robust constrained MDP problem (Eq. \eqref{eqn: original objective}). This fundamental result is shown in Lemma \ref{lemma: optimality and feasibility of F policy}. In Lemma ~\ref{lemma: optimality and feasibility of F policy}, we have two cases: with and without the slackness assumption (Assumption~\ref{assumption: slackness}). 
\begin{assumption}[Slackness Assumption]
    \label{assumption: slackness}
    We assume that \(\max_{i \in \left[1, I\right]} g_{\mathcal{P}}^{\pi^*, i} - b_i \leq -\zeta\), for some \(\zeta > 0\).
\end{assumption}

Assumption~\ref{assumption: slackness} allows us to ensure exact feasibility of the optimal policy that minimizes \(F_{\mathcal{P}}^\pi\) instead of \(\epsilon\)-feasibility. Additionally, as shown in the proof of Lemma~\ref{lemma: optimality and feasibility of F policy}, it allows us to decouple \(\lambda\) from \(\epsilon\), which improves the sample complexity.
\begin{lemma}
    \label{lemma: optimality and feasibility of F policy}
    If \(\hat{\pi}^*\) is the optimal policy of Eq. \eqref{eq: ganguly}, then \(\hat{\pi}^*\) is an \(\frac{\epsilon}{2}\)-feasible policy and \(\frac{\epsilon}{2}\)-optimal to the optimal policy \(\pi^*\) of Eq. \eqref{eqn: original objective}, when $\lambda = 4/\max\{\epsilon,\zeta\}$.
\end{lemma}

We show that our formulation's optimal policy cost objective is $\epsilon$-close to the cost objective of the policy optimized for the original RCMDP problem. Furthermore, we are able to ensure that any constraint is violated only by $\epsilon$ at the maximum. The proof utilizes the general properties of the objective to show $\epsilon$-optimality and then leverages a proof by contradiction to show that we achieve $\epsilon$-feasibility with and without slackness. Assuming a constraint is violated by more than $\epsilon$ is shown to contradict the premise that $\hat{\pi}^*$ is the optimal policy for our objective. 

Next, we need a way to attain the optimal policy $\hat{\pi}^*$ through gradient descent, which requires the (sub)gradient of the objective $F_{\mathcal{P}}^{\pi}$.

\begin{lemma}
    \label{lemma: Q function derivation}
    We can rewrite \(\nabla F_{\mathcal{P}}^{\pi}\) as
    \begin{align}
        \begin{split}
        \nabla F_{\mathcal{P}}^{\pi}(s, a)  = \tilde{d}_{\mathcal{P}}^{\pi} Q_{\mathcal{P}}^{\pi}(s, a) = d_{\mathcal{P}}^{\pi, i_{\max}^\pi} Q_{\mathcal{P}}^{\pi, i_{\max}^\pi}(s,a)
        \end{split}
    \end{align}
\end{lemma}
The proofs of Lemma 2 and 3 are given in Appendix \ref{appendix: method}. 

We now possess the required tools and provide the proposed algorithm in Algorithm~\ref{alg: acrcac}.



Each iteration performs a gradient-based policy update, but the core challenge lies in estimating the gradient itself. Since the theoretical subgradient (Lemma \ref{lemma: subgradient definition}) is proportional to the robust Q function, estimating this Q function becomes the most critical task. This is a non-trivial task because the robust Bellman operator is not a contraction under standard norms. Our Actor-Critic framework becomes vital for this, where the critic's role is to produce a stable Q-function estimate. 
\begin{enumerate}
    \item In each iteration of our algorithm, we calculate estimates \(g_{N}^{\pi_t, i}\) and \(V_{N}^{\pi_t, i}\) for the worst case average cost and worst case value function respectively by running Algorithm~\ref{alg:robust_avg_reward_td} (our critic) for \(N = O(\epsilon^{-2})\) iterations for each of the cost and constraints.
    \item Next, we compute the worst-case value function (or support function, $\hat{\sigma}_\mathcal{P}$) over the uncertainty set (Contamination, TV and Wasserstein) via Algorithm~\ref{alg:truncated_mlmc}, which implements a variance-reducing Truncated MLMC estimator.
    \item Finally, we apply the worst-case Bellman operator (Theorem \ref{theorem: robust-bellman-op-main}) to obtain the $Q$-function for each of these components. 
\end{enumerate}

With the Q-function for each component estimated, the actor performs the policy update. It uses 
Lemma \ref{lemma: Q function derivation} to identify the active objective (the most violated constraint or the cost) and selects its corresponding Q-function for the update step .

 It is to be noted that since the task of policy evaluation is identical in both constrained and unconstrained settings, we can directly employ these established, sample-efficient algorithms(\ref{alg:robust_avg_reward_td}, \ref{alg:truncated_mlmc}) for our critic. Algorithms \ref{alg:robust_avg_reward_td} and \ref{alg:truncated_mlmc} are presented in Appendix \ref{section: appendix-algos}.

\begin{algorithm}[ht]
\caption{Average-Cost Robust Constrained Actor Critic}
\label{alg: acrcac}
\begin{algorithmic}[1]
  \State \textbf{Input:} Initial policy $\pi_{0}$; iterations $T$; learning rate $\eta$
  \For{$t=0,1,\dots,T-1$}
    \State Robust evaluation: estimate $g_{N}^{\pi_t, i}$, $V_{N}^{\pi_t, i}(s,a)$ using Algorithm~\ref{alg:robust_avg_reward_td} for \(i = 0, 1, \cdots I\).
    \State Obtain \(\hat{\sigma}_{\mathcal{P}^a_s}\left(V_{N}^{\pi_t, i}\right)\) using Algorithm~\ref{alg:truncated_mlmc} for \(i = 0, 1, \cdots I\)
    
    \For{\((s, a) \in \mathcal{S} \times \mathcal{A}\)}
        \State \(\hat{Q}_{\mathcal{P}}^{\pi_t, 0}(s, a) = r(s, a) - g_{N}^{\pi_t, 0} + \hat{\sigma}_{\mathcal{P}_{s}^a}\left(V_{N}^{\pi_t, 0}\right)\)
        \For{\(i \in \{1, 2, \cdots I\}\)}
            \State \(\hat{Q}_{\mathcal{P}}^{\pi_t, i}(s, a) = c_i(s, a) - g_{N}^{\pi_t, i} + \hat{\sigma}_{\mathcal{P}^a_s}\left(V_{N}^{\pi_t, i}\right)\)
        \EndFor
    \EndFor
    \State Calculate \(\hat{Q}_{\mathcal{P}}^{\pi_t}\) from \(g_{N}^{\pi_t, i}\), \(\hat{Q}_{\mathcal{P}}^{\pi_t, i}\), \text{for} \(i \in \{0, 1, \cdots I\}\)
    \State \(\pi_{t+1} \gets \argmin_{p \in \Delta(\mathcal{A})} \{\eta \left\langle \hat{Q}_{\mathcal{P}}^{\pi_t},  p\right\rangle + \|p - \pi_t(\cdot | s)\|^2\}\)
  \EndFor
  \State \Return $\hat{\pi} = \arg \min _{t=0,\cdots, T-1} F_\mathcal{P}^{\pi_t}$
\end{algorithmic}
\end{algorithm}
\section{Theoretical Analysis}
\label{sec: theoretical analysis}
The critic's role is to estimate the Q-function, $Q_\mathcal{P}^{\pi_t}$. Per the analysis of \cite{xu2025efficient}, Lemma \ref{lemma: Q function estimation Error} (Appendix \ref{appendix: theoretical-results}) provides a guarantee that we can obtain an $\varepsilon$ accurate estimate of this Q-function with $\tilde{O}(\epsilon^{-2})$ number of samples.

A common method to prove the convergence of policy optimization methods is to form an average of the performance differences between the current policy and the optimal policy,
\begin{align*}
    \argmin_{t = 0, 1, \dots, T} \left(F_{\mathcal{P}}^{\pi_t} - F_{\mathcal{P}}^{\hat{\pi}^*}\right) \leq  \frac{1}{T} \sum_{t=0}^{T-1} \left(F_{\mathcal{P}}^{\pi_{t+1}} - F_{\mathcal{P}}^{\hat{\pi}^*}\right),
\end{align*}
which allows us to find an upper bound on the performance difference.
However, the performance difference lemma (Lemma~\ref{lemma: pdl-sun} in Appendix \ref{appendix: theoretical-results}) given by \cite{NEURIPS2024_1f28e934} expresses each difference in terms of an expectation under the stationary distribution \(d_{\mathcal{P}}^{\pi_t}\) of the current policy. Since the stationary distribution (and hence the expectation) changes with \(\pi_t\) at every step, these terms do not align across iterations and our desired telescoping structure breaks down. 

To overcome this difficulty, we introduce a regularity assumption linking performance gaps under the worst-case transition kernel \(\mathcal{P}\) to those under the nominal kernel \(P^\circ\):
\begin{assumption}
    \label{assumption: performance-gap-assumption}
    For all policies \(\pi \in \Pi\)) 
    \begin{align}
        g_{\mathcal{P}}^{\pi}-g_{\mathcal{P}}^{\hat{\pi}^*}
        &\leq 
        C \mathbb{E}_{s \sim d_{P^\circ}^\pi} \left[\left\langle Q_{P}^{\pi}(s, \cdot), \pi(\cdot | s) - \hat{\pi}^*(\cdot | s)\right\rangle\right]
    \end{align}
\end{assumption}
This assumption extends the robust performance difference lemma by relating the worst-case performance gap to the nominal kernel’s stationary distribution. Intuitively, it asserts that the degradation in performance under the worst-case model cannot exceed that under the nominal model by more than a constant multiplicative factor \(C\). A related assumption is common in the discounted robust MDP setting (\cite{tamar2014scaling,zhou2023natural,ganguly2025efficient}), which states: $\gamma p(s'|s,a) = \beta p_0(s'|s,a)$ for some $\beta \in (0,1)$. We notice that if $\gamma =1$, the assumption does not hold anymore (a trivial counterexample is when $s'=s$). Thus, we cannot leverage an assumption of the same form for our average cost setting. 
However, it is to be noted that our assumption is not completely arbitrary and is grounded in the assumption made by the discounted RMDP literature. A detailed equivalence relation is provided in Appendix \ref{appendix: theoretical-results}. 
We now have the required tools to state and prove our main theorem on convergence:

\begin{theorem}
    \label{theorem: sample complexity}

        Using a stepsize of \(\eta = O(\epsilon)\), Algorithm~\ref{alg: acrcac} returns a policy \(\hat{\pi}\) that is both \(\epsilon\)-feasible and \(\epsilon\)-optimal after $T=\tilde{O}(\epsilon^{-2}\lambda^2) $ iterations. 
\end{theorem}

We know from \cite{xu2025efficient} that the critic requires $\tilde{O}(\epsilon^{-2})$ samples. Therefore, if we assume the slackness condition, we obtain an iteration complexity of $T = O(\epsilon^{-2}\zeta^{-2})$ and a corresponding sample complexity of $O(\epsilon^{-4}\zeta^{-2})$. If we do not assume slackness, we obtain $T = O(\epsilon^{-4})$ and a sample complexity of $O(\epsilon^{-6})$.

Theorem~\ref{theorem: sample complexity} establishes that Algorithm~\ref{alg: acrcac} converges to a policy that is simultaneously \(\epsilon\)-optimal and \(\epsilon\)-feasible with sample complexities comparable to existing discounted reward settings.


The convergence analysis is composed of three key steps:
\begin{itemize}
    \item The proof begins by invoking our critical Assumption \ref{assumption: performance-gap-assumption}. This step allows us to shift the analysis to the more tractable non-robust setting where the stationary distribution is fixed with respect to a single optimal policy. We then utilize the standard result of the non-robust performance difference lemma, yielding an expectation over the Q-function. 
    \item We decompose the inner product into two parts corresponding to the policy change in one step and the remaining gap to the optimal policy. We then leverage the three-point lemma of Bregman Divergence and Holder's inequality to obtain a telescoping sum of the form
    \begin{align*}
    \begin{split}
        \cdots &\leq \frac{1}{\eta} ( \|\hat{\pi}^*(\cdot | s) - \pi_t(\cdot | s)\|^2 
     - \|\hat{\pi}^*(\cdot | s) - \pi_{t+1}(\cdot | s)\|^2 ) \\
     &+ \text{error-terms}
    \end{split}
    \end{align*}
    \item Finally, we sum the inequality over all iterations $t=0,\cdots,T-1$. The telescoping terms cancel out, leaving a bound on the average performance gap that depends on the initial policy distance, the step size $\eta$, and the critic estimation error $\varepsilon$. Using the fact that the minimum of a distribution is at most the average, we obtain the above mentioned iteration complexities. 
\end{itemize}

\if 0\section{Future Work}
There are three primary avenues of future research in the Average Cost RCMDP setting. First, research can take up smoothed approximate versions of our point-wise maximum objective, which eliminates the need to switch objectives every iteration. However, there are multiple problems with this approach, including tracking the approximation error and the challenge of deriving an appropriate performance difference lemma that does not blow up the sample complexity. Second, one can look at eliminating or weakening the assumptions made, while maintaining the same sample complexity. This is a significant challenge and would require new theoretical tools and approaches to solve. Third, there is a gap in the sample complexity between the robust constrained and robust unconstrained works (both average and discounted). Another avenue of research could tackle deriving better bounds or rigorously proving that a biquadratic sample complexity is the theoretical lower bound for RCMDPs.
\fi

\section{Conclusion}
In this work, we present an actor critic algorithm to solve the robust constrained average cost MDP problem. We show that our algorithm outputs an $\epsilon$-feasible and $\epsilon$-optimal policy with a sample complexity of $O(\epsilon^{-4})$ when using the slackness assumption and $O(\epsilon^{-6})$ when not using the slackness assumption. Not only are we the first algorithm to tackle this specific setting, but we also obtain equal sample complexity guarantees with existing discounted RCMDP algorithms.

Weakening the assumptions in this work while preserving the same sample complexity is an important avenue for future research. Furthermore, a gap persists between the sample complexity results for robust constrained and robust unconstrained settings in both average- and discounted-reward cases. Closing this gap, either through the development of improved algorithms and refined analyses or by establishing tighter lower bounds on sample complexity, constitutes another significant direction for future work.

\bibliography{aistats}

\clearpage
\appendix
\thispagestyle{empty}

\onecolumn
\section{Missing Algorithms from Section \ref{sec: method}}\label{section: appendix-algos}
\begin{algorithm}[h]
\caption{Robust average cost TD (Algorithm 2 in \cite{xu2025finite})}
\label{alg:robust_avg_reward_td}
\begin{algorithmic}[1]
\State \textbf{Input:} Policy $\pi$, Initial values $V_0$, $g_0 = 0$, Stepsizes $\eta_t, \beta_t$, Max level $N_{\max}$, Anchor state $s_0 \in \mathcal{S}$
\For{$t = 0, 1, \ldots, T - 1$}
    \For{each $(s,a) \in \mathcal{S} \times \mathcal{A}$}
        \If{Contamination}
            \State Sample $\hat{\sigma}_{P_{s,a}}(V_t)$ according to Eq. \ref{eqn: contamination estimator}
        \ElsIf{TV or Wasserstein}
            \State Sample $\hat{\sigma}_{P_{s,a}}(V_t)$ according to Algorithm \ref{alg:truncated_mlmc}
        \EndIf
    \EndFor
    \State $\hat{T}_{g_0}(V_t)(s) \gets \sum_a \pi(a|s) \left[ r(s,a) - g_0 + \hat{\sigma}_{P_{s,a}}(V_t) \right], \quad \forall s \in \mathcal{S}$
    \State $V_{t+1}(s) \gets V_t(s) + \eta_t \left( \hat{T}_{g_0}(V_t)(s) - V_t(s) \right), \quad \forall s \in \mathcal{S}$
    \State $V_{t+1}(s) \gets V_{t+1}(s) - V_{t+1}(s_0), \quad \forall s \in \mathcal{S}$
\EndFor
\For{$t = 0, 1, \ldots, T - 1$}
    \For{each $(s,a) \in \mathcal{S} \times \mathcal{A}$}
        \If{Contamination}
            \State Sample $\hat{\sigma}_{P_{s,a}}(V_t)$ according to Eq. \ref{eqn: contamination estimator}
        \ElsIf{TV or Wasserstein}
            \State Sample $\hat{\sigma}_{P_{s,a}}(V_t)$ according to Algorithm \ref{alg:truncated_mlmc}
        \EndIf
    \EndFor
    \State $\delta_t(s) \gets \sum_a \pi(a|s) \left[ r(s,a) + \hat{\sigma}_{P_{s,a}}(V_T) \right] - V_T(s), \quad \forall s \in \mathcal{S}$
    \State $\bar{\delta}_t \gets \frac{1}{S} \sum_s \delta_t(s)$
    \State $g_{t+1} \gets g_t + \beta_t (\bar{\delta}_t - g_t)$
\EndFor
\State \textbf{return} $V_T, g_T$
\end{algorithmic}
\end{algorithm}
\begin{algorithm}[h]
\caption{Truncated MLMC Estimator, Algorithm 1 in \cite{xu2025finite}}
\label{alg:truncated_mlmc}
\begin{algorithmic}[1]
\State \textbf{Input:} $s \in \mathcal{S}$, $a \in \mathcal{A}$, Max level $N_{\max}$, Value function $V$
\State Sample $N \sim \mathrm{Geom}(0.5)$
\State $N' \gets \min\{N, N_{\max}\}$
\State Collect $2^{N'+1}$ i.i.d. samples of $\{ s'_i \}_{i=1}^{2^{N'+1}}$ with $s'_i \sim \tilde{P}^a_s$ for each $i$
\State $\hat{P}^{a,E}_{s, N'+1} \gets \frac{1}{2^{N'}} \sum_{i=1}^{2^{N'}} \mathbf{1}\{ s'_{2i} \}$
\State $\hat{P}^{a,O}_{s, N'+1} \gets \frac{1}{2^{N'}} \sum_{i=1}^{2^{N'}} \mathbf{1}\{ s'_{2i-1} \}$
\State $\hat{P}^a_{s, N'+1} \gets \frac{1}{2^{N'+1}} \sum_{i=1}^{2^{N'+1}} \mathbf{1}\{ s'_i \}$
\State $\hat{P}^a_{s, 1} \gets \mathbf{1}\{ s'_1 \}$
\If{TV}
    \State Obtain $\sigma_{\hat{P}^a_{s,1}}(V)$, $\sigma_{\hat{P}^a_{s, N'+1}}(V)$, $\sigma_{\hat{P}^{a,E}_{s, N'+1}}(V)$, $\sigma_{\hat{P}^{a,O}_{s, N'+1}}(V)$ from Eq. \ref{eqn: total variation sigma}
\ElsIf{Wasserstein}
    \State Obtain $\sigma_{\hat{P}^a_{s,1}}(V)$, $\sigma_{\hat{P}^a_{s, N'+1}}(V)$, $\sigma_{\hat{P}^{a,E}_{s, N'+1}}(V)$, $\sigma_{\hat{P}^{a,O}_{s, N'+1}}(V)$ from Eq. \ref{eqn: wasserstein sigma}
\EndIf
\State $\Delta_{N'}(V) \gets \sigma_{\hat{P}^a_{s, N'+1}}(V) - \frac{1}{2} \left[ \sigma_{\hat{P}^{a,E}_{s, N'+1}}(V) + \sigma_{\hat{P}^{a,O}_{s, N'+1}}(V) \right]$
\State $\hat{\sigma}_{\mathcal{P}^a_s}(V) \gets \sigma_{\hat{P}^a_{s,1}}(V) + \frac{\Delta_{N'}(V)}{\mathbb{P}(N'=n)}$, where $p'(n) = \mathbb{P}(N' = n)$
\State \Return $\hat{\sigma}_{\mathcal{P}^a_s}(V)$
\end{algorithmic}
\end{algorithm}
\clearpage
\section{Missing Proofs for Section~\ref{sec: method}}\label{appendix: method}
\begin{lemma}[Proof of Lemma ~\ref{lemma: optimality and feasibility of F policy}] 
    \begin{proof}
    Let $\hat{\pi}^*$ be the policy that minimizes our smoothed objective $F$. Then we have that the optimality difference between \(\pi^*\) (true optimal policy of Eq \ref{eqn: original objective}) and \(\hat{\pi}^*\) is
    \begin{align}
        \frac{g_{\mathcal{P}}^{\hat{\pi}^*, 0}}{\lambda} - \frac{g_{\mathcal{P}}^{\pi^*, 0}}{\lambda} &\stackrel{(a)}{\leq} \max_{i} G_{\mathcal{P}}^{\hat{\pi}^*} - \max_{i} G_{\mathcal{P}}^{\pi^*} \\
        &\stackrel{(b)}\leq 0
    \end{align}
    where \((a)\) is from the definition of \(\max\) and because \(\pi^*\) is a feasible policy, and \((b)\) is by the optimality of \(\hat{\pi}^*\). Then to prove feasibility, we have two cases.
    \paragraph{Case 1 (No slackness):}
    By contradiction, assume optimal policy $\hat{\pi}^*$ violates the constraints by more than $\frac{\epsilon}{2}$, where we set the slackness coefficient \(\zeta\) to 0:
    \begin{align}
        \label{eqn: no slackness contradiction assumption}
        \max_{i \in \{1, \cdots I\}} \left\{g_{\mathcal{P}}^{\hat{\pi}^*, i} - b_i + \zeta\right\} = \max_{i \in \{1, \cdots I\}} \left\{g_{\mathcal{P}}^{\hat{\pi}^*, i} - b_i\right\} > \frac{\epsilon}{2}
    \end{align}
    We set the hyperparameter $\lambda = \frac{4}{\epsilon}$. The maximum average cost for the objective, $g_{\mathcal{P}}^{\hat{\pi}^*, i}$, is bounded by 1 because the cost function is bounded by 1. Similar logic holds for \(\pi^*\). Therefore, the objective cost term satisfies:
    \begin{align}
        \label{eqn: no slackness optimality bound}
        \frac{g_{\mathcal{P}}^{\hat{\pi}^*, 0}}{ \lambda} \leq \frac{1}{\lambda} = \frac{\epsilon}{4}, \quad \frac{g_{\mathcal{P}}^{\pi^*, 0}}{ \lambda} \leq \frac{1}{\lambda} = \frac{\epsilon}{4}
    \end{align}
    Combining \eqref{eqn: no slackness contradiction assumption} and \eqref{eqn: no slackness optimality bound} yields
    \begin{align} 
        \label{eqn: contra1}
        \max\left\{ \frac{g_{\mathcal{P}}^{\hat{\pi}^*, 0}}{ \lambda}, \max_i\left\{g_{\mathcal{P}}^{\hat{\pi}^*, i} - b_i\right\} \right\} = \max_i G_{\mathcal{P}}^{\hat{\pi}^*}> \frac{\epsilon}{2}. 
    \end{align}
    Thus, we have
    \begin{align}
        \begin{split}
            \max_{i} G_{\mathcal{P}}^{\pi^*} = \max\left\{ \frac{g_{\mathcal{P}}^{{\pi}^*, 0}}{ \lambda}, \max_i\left\{g_{\mathcal{P}}^{{\pi}^*, i} - b_i\right\} \right\} \leq \frac{\epsilon}{4} < \frac{\epsilon}{2} = \max_i G_{\mathcal{P}}^{\hat{\pi}^*}
        \end{split}
    \end{align}
    which is a contradiction as by definition, \(G_{\mathcal{P}}^{\hat{\pi}^*} \leq G_{\mathcal{P}}^{{\pi}^*}\). Therefore, the maximum violation is at most $\epsilon/2$ with \(\lambda = \frac{4}{\epsilon}\).
    \paragraph{Case 2 (With Slackness):}
    Assumption \ref{assumption: slackness} gives us a way to prove exact feasibility of $\hat{\pi}^*$ by choosing $\lambda = \frac{4}{\zeta}$, assuming that \(\zeta > \epsilon\). We again assume by contradiction that the optimal policy \(\hat{\pi}^*\) violates the constraints by more than \(\frac{\epsilon}{2}\):
    \begin{align}
        \max_{i \in \{1, \cdots I\}} \left\{g_{\mathcal{P}}^{\hat{\pi}^*, i} - b_i + \zeta\right\} > \frac{\zeta}{2}
    \end{align}
    Then following the same logic in Case 1, we have \(\frac{g_{\mathcal{P}}^{\hat{\pi}^*, 0}}{\lambda} \leq \frac{\zeta}{4}\) and
    \begin{align} 
        \max\left\{ \frac{g_{\mathcal{P}}^{\hat{\pi}^*, 0}}{ \lambda}, \max_i\left\{g_{\mathcal{P}}^{\hat{\pi}^*, i} - b_i + \zeta\right\} \right\}  = \max_i G_{\mathcal{P}}^{\hat{\pi}^*}> \frac{\zeta}{2}. 
    \end{align}
    This yields
    \begin{align}
        \begin{split}
            \max_{i} G_{\mathcal{P}}^{\pi^*} = \max\left\{ \frac{g_{\mathcal{P}}^{{\pi}^*, 0}}{ \lambda}, \max_i\left\{g_{\mathcal{P}}^{{\pi}^*, i} - b_i\right\} \right\} \leq \frac{\zeta}{4} < \frac{\zeta}{2} = \max_i G_{\mathcal{P}}^{\hat{\pi}^*}
        \end{split}
    \end{align}
    which again is a contradiction. Thus, the original assumption is false and
    \begin{align}
        \max_{i \in \{1, \cdots I\}} \left\{g_{\mathcal{P}}^{\hat{\pi}^*, i} - b_i + \zeta\right\} \leq \frac{\zeta}{2} \implies \max_{i \in \{1, \cdots I\}} \left\{g_{\mathcal{P}}^{\hat{\pi}^*, i} - b_i\right\} \leq -\frac{\zeta}{2} \leq 0
    \end{align}
    Thus, with the slackness assumption, we obtain exact feasibility.
    \end{proof}
\end{lemma}
\begin{lemma}[Restatement of Lemma ~\ref{lemma: Q function derivation}]
    We can rewrite \(\nabla F_{\mathcal{P}}^{\pi}\) as
    \begin{align}
        \nabla F_{\mathcal{P}}^{\pi}(s, a)  = \tilde{d}_{\mathcal{P}}^{\pi} Q_{\mathcal{P}}^{\pi}(s, a) = d_{\mathcal{P}}^{\pi, i_{\max}^\pi} Q_{\mathcal{P}}^{\pi, i_{\max}^\pi}(s,a)
    \end{align}
    \begin{proof}
        We first look at the inner product between the subgradient of \(F_{\mathcal{P}}^\pi\) and \(\pi\)
        \begin{align}
            \left\langle \nabla F_{\mathcal{P}}^\pi, \pi\right\rangle &\stackrel{(a)}{=} \sum_{i=0}^I w_i^\pi \left\langle \nabla G_{\mathcal{P}}^{\pi, j} , \pi\right\rangle \\
            &\stackrel{(b)}{=} \sum_{j=0}^I w_j^\pi \sum_{s} d_{\mathcal P}^{\pi,j}(s)\sum_{a} Q_{\mathcal{P}}^{\pi, j}(s,a)\,\pi(a\mid s)
        \end{align}
        where \((a)\) uses the definition of the derivative of the objective, and in \((b)\), we use the definition of the subgradient from Lemma~\ref{lemma: subgradient definition}. Fixing a state \(s\) and looking only at the coefficient of \(\pi(a | s)\) yields
        \begin{align}
            \sum_{i=0}^I w_i^\pi d_{\mathcal P}^{\pi,i}(s) Q_{\mathcal{P}}^{\pi, i}(s,a) &= \sum_{i=0}^I w_i^\pi d_{\mathcal P}^{\pi,i}(s) Q_{\mathcal{P}}^{\pi, i}(s,a) \\
            &= \sum_{i=0}^I d_{\mathcal{P}}^{\pi, i_{\max}^\pi} 1_{i = i_{\max}^\pi} Q_{\mathcal{P}}^{\pi, i_{\max}^\pi}(s,a)\\
            &= d_{\mathcal{P}}^{\pi, i_{\max}^\pi} Q_{\mathcal{P}}^{\pi, i_{\max}^\pi}(s,a)
        \end{align}
        which is the desired result.
    \end{proof}
\end{lemma}
\section{Missing Lemmas and Proofs for Section \ref{sec: theoretical analysis}}\label{appendix: theoretical-results}

\begin{lemma}[\cite{xu2025efficient}] \label{lemma: Q function estimation Error}

We have that the expected estimation error between the true smoothed \(Q\)-function and our estimate is bounded by \(\varepsilon\) with a sample complexity of \(O(\epsilon^{-2})\).    
\begin{align}
\mathbb{E}\left[\left\|Q_{\mathcal{P}}^{\pi_t}(s, \cdot) - \hat{Q}_{\mathcal{P}}^{\pi_t}(s, \cdot)\right\|_{\infty}\right] \leq \varepsilon
\end{align}
\end{lemma}

Next, we define the performance difference lemma for robust average cost MDPs from \citep{NEURIPS2024_1f28e934}.

\begin{lemma}[Lemma 4.1 in \citep{NEURIPS2024_1f28e934}] \label{lemma: pdl-sun}
    For any two policies \(\pi, \pi'\), we have that
    \begin{align}
        g_{\mathcal{P}}^\pi - g_{\mathcal{P}}^{\pi'} \geq \mathbb{E}_{s \sim d^{\pi'}_{\mathcal{P}}}\left[\left\langle Q^{\pi}_{\mathcal{P}}(s, \cdot), \pi(\cdot | s) - \pi'(\cdot | s)\right\rangle\right]
    \end{align}
    where \(d_{\mathcal{P}}^{\pi'}\) denotes the stationary distribution under the worst-case transition kernel of policy \(\pi'\).
\end{lemma}

\subsection{A detailed motivation of Assumption \ref{assumption: performance-gap-assumption}}\label{appendix: motivation-for-assumption}

We draw a direct connection between our Assumption~\ref{assumption: performance-gap-assumption} and the corresponding assumption commonly employed in the \textit{discounted robust MDP} literature.

\begin{itemize}
    \item In the discounted setting, \textbf{Lemma B.3} of \cite{ganguly2025efficient} is central to the convergence analysis and follows as a \textit{direct consequence} of the assumption 
    \(\gamma p(s' \mid s, a) \le \beta p_0(s' \mid s, a)\) for all \(s', s, a\), where \(\beta \in (0, 1)\). 
    This multiplicative dominance condition induces a contraction in the state-transition dynamics, which leads to the bound
    \begin{equation}\label{eq:lemma-b3-njit}
        \Phi(\pi) - \Phi(\pi^*)
        \;\le\;
        \frac{1}{1 - \beta}\,
        \mathbb{E}_{s \sim d_{P^\circ}^{\pi^*}}
        \Big[
            \big\langle
                Q_{\mathcal{P}}^{\pi_t}(s, \cdot),\,
                \pi_t(\cdot \mid s) - \pi^*(\cdot \mid s)
            \big\rangle
        \Big],
    \end{equation}
    where \(\Phi\) denotes their discounted objective function. 
    The factor \(\tfrac{1}{1-\beta}\) quantifies the effective discount-induced dependence between states.

    \item In contrast, in the \textit{average reward} robust MDP setting there is no discount factor, making the bound in \eqref{eq:lemma-b3-njit} inapplicable since no contraction holds. 
    To address this, we introduce Assumption~\ref{assumption: performance-gap-assumption}, which serves as an analogous regularity condition:
    \begin{equation}\label{eq:avg-reward-assumption}
        g_{\mathcal{P}}^{\pi} - g_{\mathcal{P}}^{\hat{\pi}^*}
        \;\le\;
        C\,\mathbb{E}_{s \sim d_{P^\circ}^{\pi}}
        \Big[
            \big\langle
                Q_{P}^{\pi}(s, \cdot),\,
                \pi(\cdot \mid s) - \hat{\pi}^*(\cdot \mid s)
            \big\rangle
        \Big].
    \end{equation}

    Here, the constant \(C \ge 1\) replaces the role of the geometric factor \(\tfrac{1}{1-\beta}\) from the discounted setting. 
    Intuitively, \(C\) captures the extent to which the stationary distribution under the worst-case transition kernel \(\mathcal{P}\) can differ from that under the nominal kernel \(P^\circ\), thus providing a measure of distributional regularity in the absence of discounting.

    This substitution generalizes the discounted assumption to the average-reward regime by replacing the explicit contraction (through \(\beta\)) with a bounded performance coupling constant \(C\), ensuring that the worst-case performance degradation remains controlled relative to the nominal dynamics.
\end{itemize}

\begin{lemma}[Restatement of Theorem 4]
    Using a stepsize of \(\eta = O(\epsilon)\), Algorithm~\ref{alg: acrcac} returns a policy \(\hat{\pi}\) that is both \(\epsilon\)-feasible and \(\epsilon\)-optimal after $T=\frac{18C^2Q_{\max}^2\Delta}{\epsilon^2}\lambda^2 $ iterations. 
\end{lemma}

\begin{proof}


By the non-robust performance difference lemma between the optimal policy and the current policy, we have
    \begin{align}
    \label{eqn: optimal bound}
    F_{\mathcal{P}}^{\pi_t} - F_{\mathcal{P}}^{\hat{\pi}^*} &= g_{\mathcal{P}}^{\pi_t, i} - g_{\mathcal{P}}^{\hat{\pi}^*, i} \\
    &\leq g_{P_{\pi_t}}^{\pi_t, i} - g_{P_{\pi_t}}^{\hat{\pi}^*, i} \\
    &\stackrel{(a)}{=} \mathbb{E}_{s \sim d_{P^\circ}^{\hat{\pi}^*}}\left[\left\langle Q_{\mathcal{P}}^{\pi_t}(s, \cdot), \pi_{t}(\cdot | s) - {\hat{\pi}^*}(\cdot | s)\right\rangle \right]\\
    &= \mathbb{E}_{s \sim d_{P^\circ}^{\hat{\pi}^*}}\left[\left\langle Q_{\mathcal{P}}^{\pi_t}(s, \cdot), \pi_{t}(\cdot | s) - {\pi_{t+1}}(\cdot | s)\right\rangle \right] + \mathbb{E}_{s \sim d_{P^\circ}^{\hat{\pi}^*}}\left[\left\langle Q_{\mathcal{P}}^{\pi_{t}}(s, \cdot), \pi_{t+1}(\cdot | s) - {\hat{\pi}^*}(\cdot | s)\right\rangle \right]
    \end{align}

where \((a)\) uses Assumption \ref{assumption: performance-gap-assumption}. From the three point lemma of Bregman Divergence
\begin{align}\label{eq: optimality-three-point-1}
\begin{split}
     F_{\mathcal{P}}^{\pi_t} - F_{\mathcal{P}}^{\hat{\pi}^*} 
     &\leq 
     \mathbb{E}_{s \sim d_{P^\circ}^{\hat{\pi}^*}}
     \Big[
     \left\langle Q_{\mathcal{P}}^{\pi_t}(s, \cdot), \pi_{t}(\cdot | s) - {\pi_{t+1}}(\cdot | s)\right\rangle 
     - \tfrac{1}{\eta}\|\pi_{t+1}(\cdot | s) - \pi_t(\cdot | s)\|^2
     \\
     &\hspace{2cm}
     + \tfrac{1}{\eta}\|\hat{\pi}^*(\cdot | s) - \pi_t(\cdot | s)\|^2 
     - \tfrac{1}{\eta}\|\hat{\pi}^*(\cdot | s) - \pi_{t+1}(\cdot | s)\|^2 +2\varepsilon \Big]
\end{split}
\end{align}

From Holder's inequality, we have
\begin{align}
    \begin{split}
    \langle Q_{\mathcal{P}}^{\pi_t}(s, \cdot), \pi_{t}(\cdot | s) - {\pi_{t+1}}(\cdot | s)\rangle - \tfrac{1}{\eta}\|\pi_{t+1}(\cdot | s) - \pi_t(\cdot | s)\|^2 &\leq \|Q_{\mathcal{P}}^{\pi_t}(s, \cdot)\|_{\infty} \|\pi_{t}(\cdot | s) - {\pi_{t+1}}(\cdot | s)\|_1 \\ 
    &- \tfrac{1}{2\eta}\|\pi_{t+1}(\cdot | s) - \pi_t(\cdot | s)\|_2^2
    \end{split}
\end{align}

By adding and subtracting $\frac{\eta}{2}\|Q_{\mathcal{P}}^{\pi_t}(s, \cdot)\|_{\infty}^2$ and using the fact that $\|\pi(\cdot | s) - \pi'(\cdot | s)\|_1 \geq \|\pi(\cdot | s) - \pi'(\cdot | s)\|_2$, we get

\begin{align}
    \begin{split}
    \langle Q_{\mathcal{P}}^{\pi_t}(s, \cdot), \pi_{t}(\cdot | s) - {\pi_{t+1}}(\cdot | s)\rangle - \tfrac{1}{\eta}\|\pi_{t+1}(\cdot | s) - \pi_t(\cdot | s)\|^2 &\leq \frac{-1}{2\eta}(\eta\|Q_{\mathcal{P}}^{\pi_t}(s, \cdot)\|_{\infty} - \|\pi_{t+1}(\cdot | s) - \pi_t(\cdot | s)\|)^2 \\&+ \frac{\eta}{2}\|Q_{\mathcal{P}}^{\pi_t}(s, \cdot)\|_{\infty}^2 \\
    &\leq \frac{\eta}{2}\|Q_{\mathcal{P}}^{\pi_t}(s, \cdot)\|_{\infty}^2
    \end{split}
\end{align}

Now summing equation \ref{eq: optimality-three-point-1} over t and taking the average, we have
\begin{align}
    \sum_t (F_{\mathcal{P}}^{\pi_t} - F_{\mathcal{P}}^{\hat{\pi}^*}) &\leq  \sum_{t=0}^{T-1} \left(\mathbb{E}_{s \sim d_{P^\circ}^{\hat{\pi}^*}}\frac{\eta}{2}\|Q_{\mathcal{P}}^{\pi_t}(s, \cdot)\|_{\infty}^2 + \frac{1}{\eta} \mathbb{E}_{s \sim d_{P^\circ}^{\hat{\pi}^*}} \left[ \|\hat{\pi}^*(\cdot | s) - \pi_t(\cdot | s)\|^2 - \|\hat{\pi}^*(\cdot | s) - \pi_{t+1}(\cdot | s)\|^2\right] + 2\varepsilon\right) \\
     &\leq \frac{CT\eta Q_{\max}^2}{2} + \frac{C}{\eta} \mathbb{E}_{s \sim d_{P^\circ}^{\hat{\pi}^*}} \|\hat{\pi}^*(\cdot | s) - \pi_0(\cdot | s)\|^2 + 2CT\varepsilon
\end{align}
where \(C\) is the distribution mismatch coefficient. Let $\hat{\pi}$ be the output of our algorithm where $\hat{\pi} = \arg \min _{t=0,\cdots, T-1} F_\mathcal{P}^{\pi_t}$. Then we have
\begin{align}
    F_\mathcal{P}^{\hat{\pi}} - F_\mathcal{P}^{\hat{\pi}^*} &\leq \frac{1}{T} \sum_t (F_{\mathcal{P}}^{\pi_t} - F_{\mathcal{P}}^{\hat{\pi}^*})\\
    &\leq \frac{C\eta Q_{\max}^2}{2} + \frac{C}{T\eta} \mathbb{E}_{s \sim d_{P^\circ}^{\hat{\pi}^*}} \|\hat{\pi}^*(\cdot | s) - \pi_0(\cdot | s)\|^2 + 2C\varepsilon \leq \frac{\epsilon}{2}
\end{align}

First, we note that $\varepsilon$ can be made arbitrarily small. From \cite{xu2025efficient} and running the critic estimate (Algorithm \ref{alg:robust_avg_reward_td}) with $O(\epsilon^{-2})$ samples, we can obtain $\varepsilon = O(\epsilon)$. 

Solving for \(T\), by equating each of the three terms to $\epsilon/6$,  yields \(T = \frac{18C^2Q_{\max}^2\Delta}{\epsilon^2}\) 
using a step size \(\eta = \frac{\epsilon}{2CQ_{\max}^2}\), where $\Delta = \|\hat{\pi}^*(\cdot | s) - \pi_0(\cdot | s)\|^2$.

However, to achieve both \(\epsilon\)-optimality and \(\epsilon\)-feasibility, we must run the algorithm until ${g_{\mathcal{P}}^{\hat{\pi}, 0}} - {g_{\mathcal{P}}^{\hat{\pi}^*, 0}} \leq \epsilon$. Now, we have 
\begin{align}
    \frac{g_{\mathcal{P}}^{\hat{\pi}, 0}}{\lambda} - \frac{g_{\mathcal{P}}^{\hat{\pi}^*, 0}}{\lambda} \leq F_{\mathcal{P}}^{\hat{\pi}} - F_{\mathcal{P}}^{\hat{\pi}^*} \leq  \frac{\epsilon}{\lambda}
\end{align}

Thus, number of iterations needed is \(\frac{18C^2Q_{\max}^2\Delta}{\epsilon^2}\lambda^2\).

\end{proof}

\section{Numerical Experiments}
\begin{figure}[htbp]
    \centering
    \includegraphics[width=0.9\textwidth]{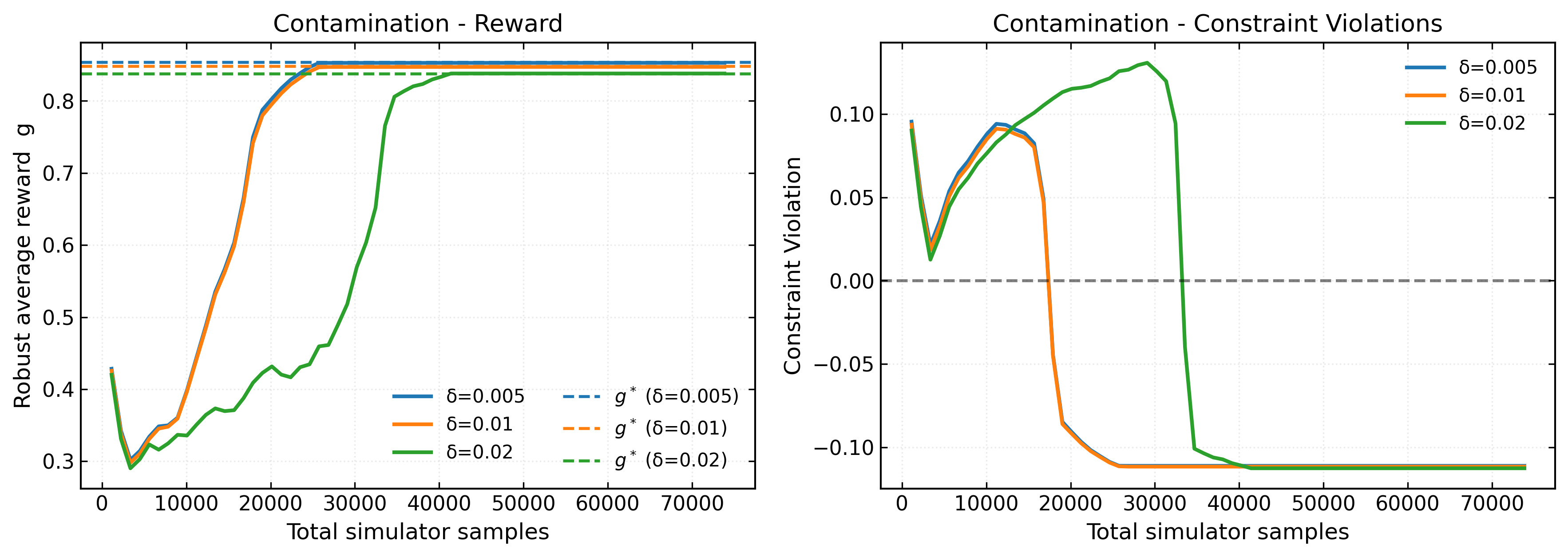}
    \caption{Performance of the Robust Constrained Average-Cost Actor-Critic algorithm under the Contamination uncertainty set.}
    \label{fig:contamination}
\end{figure}

\begin{figure}[htbp]
    \centering
    \includegraphics[width=0.9\textwidth]{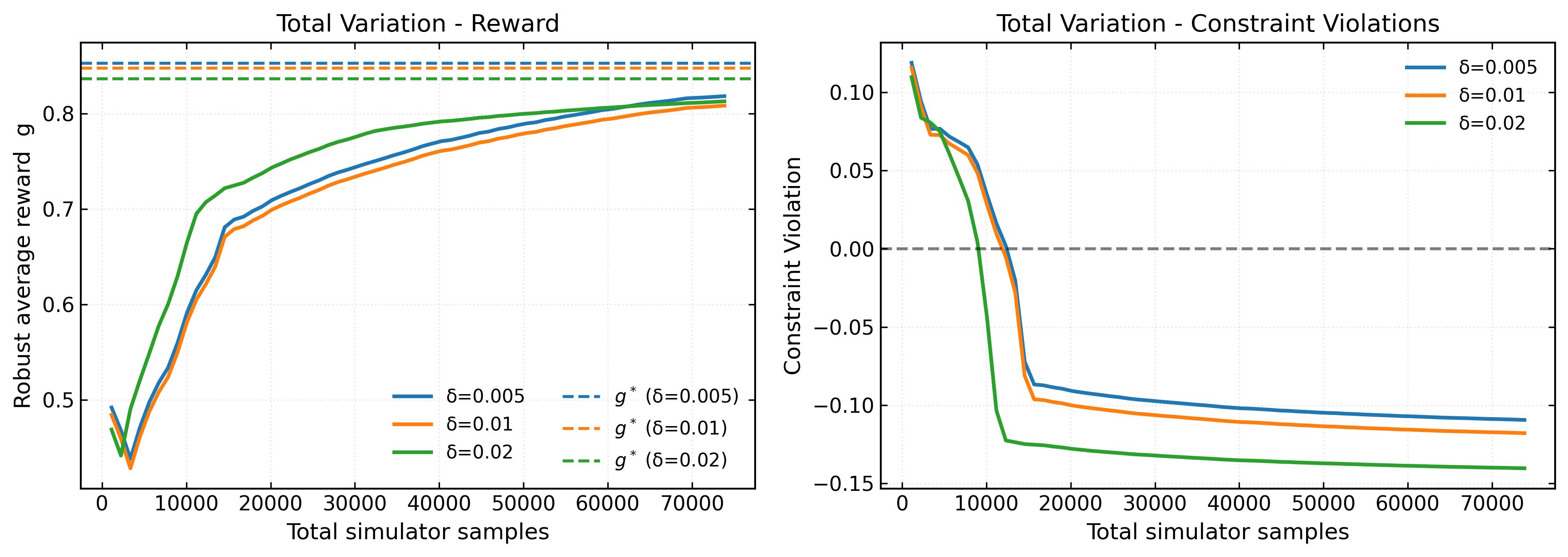}
    \caption{Performance of the Robust Constrained Average-Cost Actor-Critic algorithm under the Total Variation (TV) uncertainty set.}
    \label{fig:tv}
\end{figure}

\begin{figure}[htbp]
    \centering
    \includegraphics[width=0.9\textwidth]{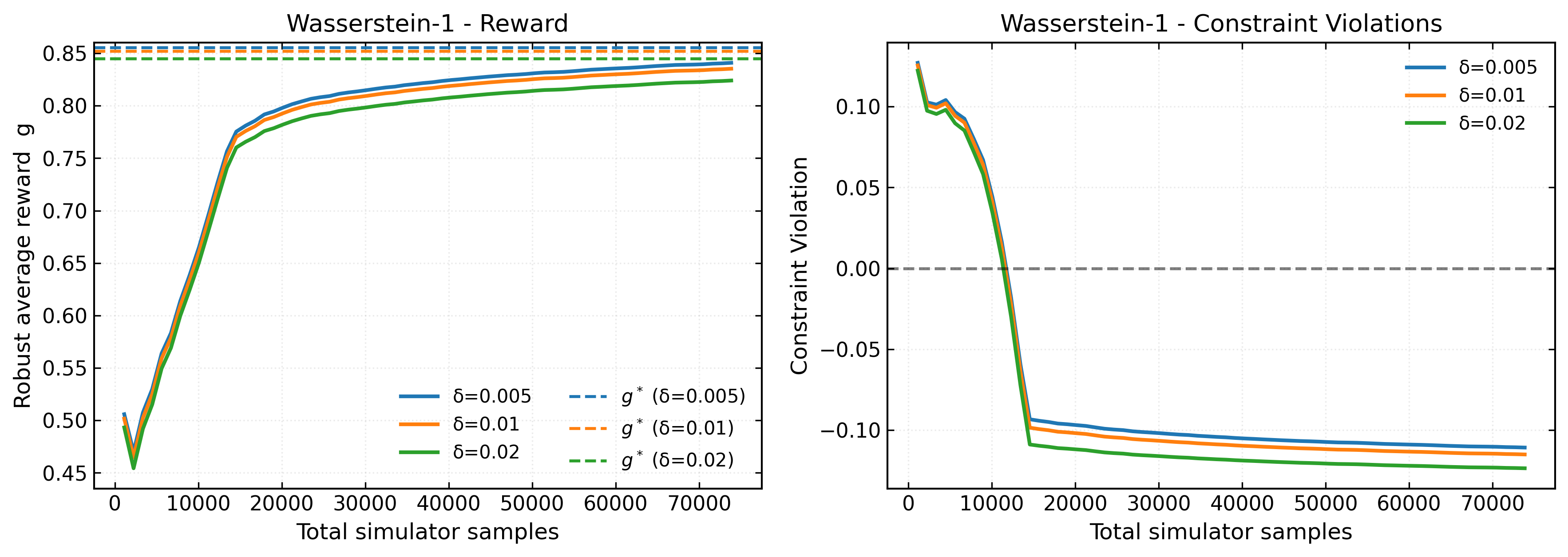}
    \caption{Performance of the Robust Constrained Average-Cost Actor-Critic algorithm under the Wasserstein uncertainty set.}
    \label{fig:wasserstein}
\end{figure}

\begin{figure}[htbp]
    \centering
    \includegraphics[width=0.9\textwidth]{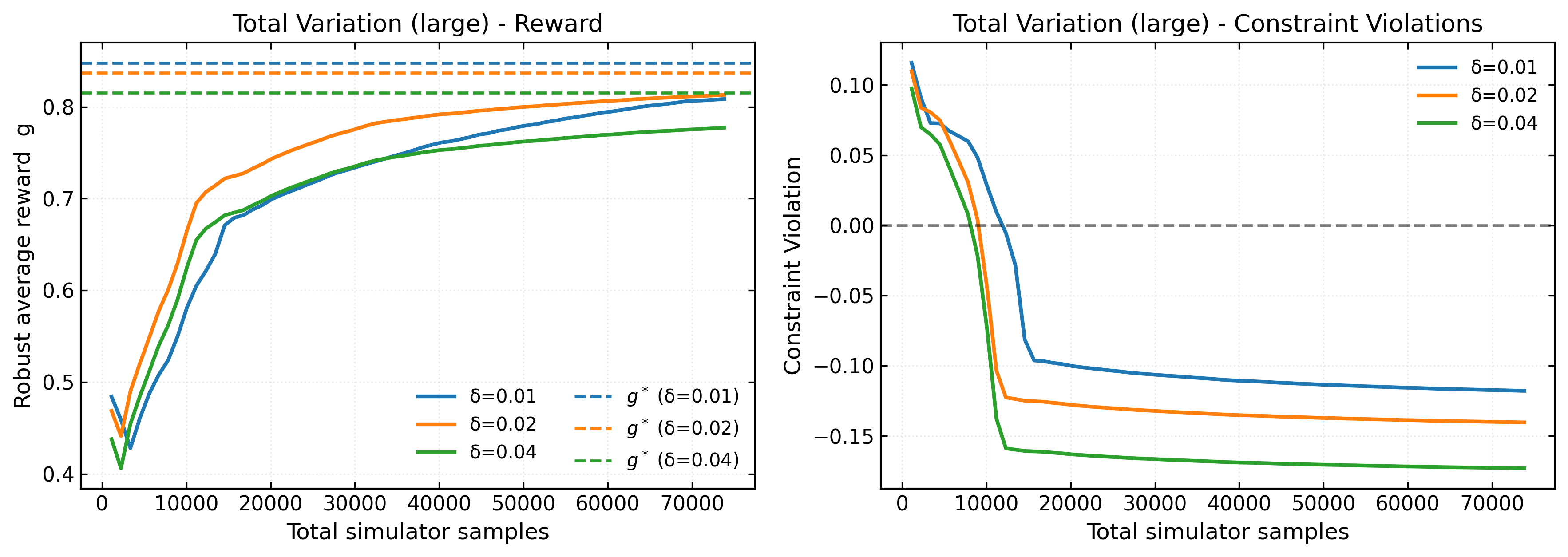}
    \caption{Performance of the Robust Constrained Average-Cost Actor-Critic algorithm under the TV uncertainty set with larger uncertainty set radii.}
    \label{fig:tv_large}
\end{figure}

\begin{figure}[htbp]
    \centering
    \includegraphics[width=0.9\textwidth]{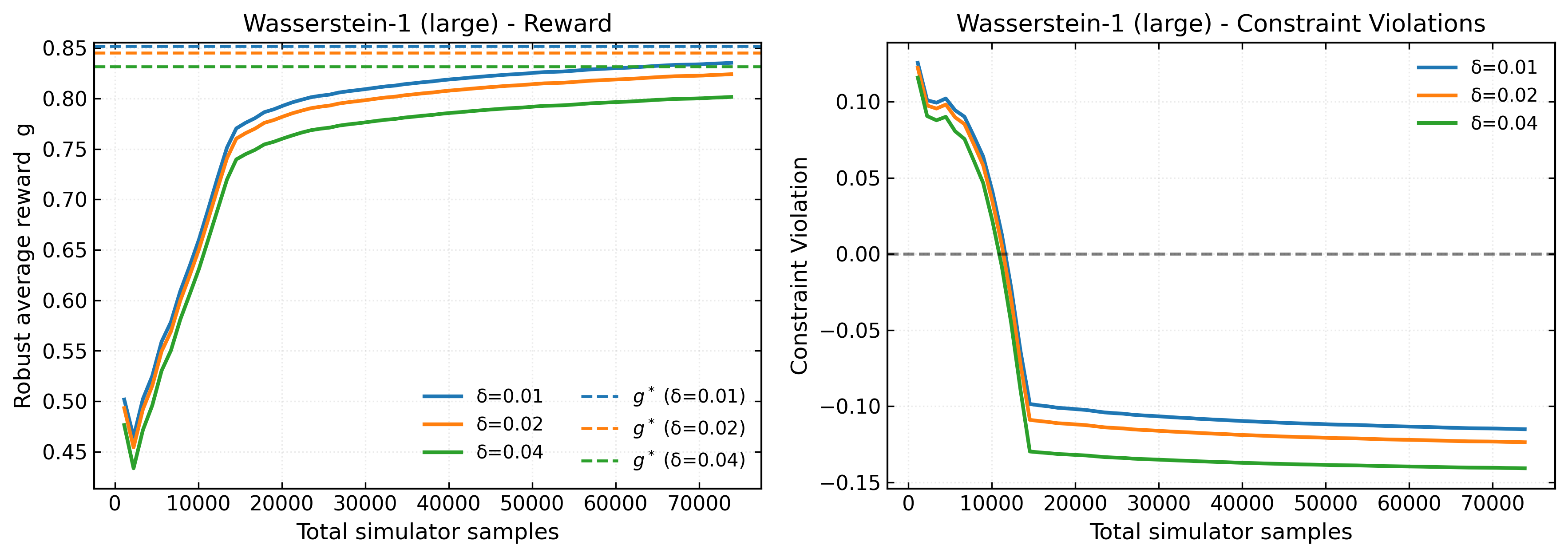}
    \caption{Performance of the Robust Constrained Average-Cost Actor-Critic algorithm under the Wasserstein uncertainty set with larger uncertainty set radii.}
    \label{fig:wasserstein_large}
\end{figure}

\end{document}